\documentclass[10pt,twocolumn,letterpaper]{article}

\usepackage{cvpr}
\usepackage{times}
\usepackage{epsfig}
\usepackage{graphicx}
\usepackage{amsmath}
\usepackage{amssymb}

\usepackage[utf8]{inputenc} 
\usepackage[T1]{fontenc}    
\usepackage{url}            
\usepackage{booktabs}       
\usepackage{amsfonts}       
\usepackage{nicefrac}       
\usepackage{microtype}      
\usepackage[amssymb]{SIunits}
\usepackage{graphicx}
\usepackage{amsmath}
\usepackage{amssymb}
\usepackage{amsthm}
\usepackage{multirow}
\usepackage{enumerate}
\usepackage{enumitem}
\usepackage{wrapfig}
\usepackage{pifont}%

\graphicspath{%
	{./figures/}%
}

\usepackage[breaklinks=true,bookmarks=false]{hyperref}

\cvprfinalcopy 

\def\cvprPaperID{257} 

\ifcvprfinal\fi
\newtheorem{theorem}{Theorem}[section]

\newcommand{\kernel}[2]{\kappa(\mathbf{#1},#2)}

\newcommand{\kerv}{kervolution }

\newcommand{\fref}[1]{Figure~\ref{#1}}
\newcommand{\sref}[1]{Section~\ref{#1}}
\newcommand{\tref}[1]{Table~\ref{#1}}

\hypersetup{pdftitle={Kervolutional Neural Networks},pdfauthor={Chen Wang, Jianfei Yang, Lihua Xie, Junsong Yuan},pdfsubject={Computer Vision and Pattern Recognition},pdfkeywords={Kervolution, Convolution, Neural Networks}}

\title{\href{https://github.com/wang-chen/kervolution}{Kervolutional Neural Networks}}

\author{\href{https://wang-chen.github.io}{Chen Wang}$^{1}$\\
    {\tt\small \href{mailto:chenwang@dr.com}{chenwang@dr.com}}
    \and
    Jianfei Yang$^{1}$\\
    {\tt\small \href{mailto:yang0478@ntu.edu.sg}{yang0478@ntu.edu.sg}}
    \and
    Lihua Xie$^{1}$\\
    {\tt\small \href{mailto:elhxie@ntu.edu.sg}{elhxie@ntu.edu.sg}}
    \and
    Junsong Yuan$^{2}$\\
    {\tt\small  \href{mailto:jsyuan@buffalo.edu}{jsyuan@buffalo.edu}}
    \and
    $^{1}$School of Electrical and Electronic Engineering, Nanyang Technological University, Singapore
    \and
    $^{2}$Computer Science and Engineering Department, State University of New York at Buffalo, USA\\
}

\pdfoutput=1

\usepackage{pdfpages}

\begin{document}

\maketitle

\begin{abstract}
Convolutional neural networks (CNNs) have enabled the state-of-the-art performance in many computer vision tasks.
However, little effort has been devoted to establishing convolution in non-linear space.
Existing works mainly leverage on the activation layers, which can only provide point-wise non-linearity.
To solve this problem, a new operation, kervolution (\textbf{ker}nel con\textbf{volution}), is introduced to approximate complex behaviors of human perception systems leveraging on the kernel trick.
It generalizes convolution, enhances the model capacity, and captures higher order interactions of features, via patch-wise kernel functions, but without introducing additional parameters.
Extensive experiments show that kervolutional neural networks (KNN) achieve higher accuracy and faster convergence than baseline CNN.
\end{abstract}

\section{Introduction}\label{sec:introduction}

Convolutional neural networks (CNNs) have been tremendously successful in computer vision, \eg image recognition \cite{Krizhevsky:2012wl,He:2016ib} and object detection \cite{Girshick:2015id,Ren:2017kt}.
The core operator, convolution, was partially inspired by the animal visual cortex where different neurons respond to stimuli in a restricted and partially overlaped region known as the receptive field \cite{hubel1962receptive,hubel1968receptive}.
Convolution leverages on its equivariance to translation to improve the performance of a machine learning system \cite{Goodfellow2016deep}.
Its efficiency lies in that the learnable parameters are sparse and shared across the entire input (receptive field).
Nonetheless, convolution still has certain limitations, which will be analyzed below. To address them, this paper introduces kervolution to generalize convolution via the kernel trick. The artificial neural networks containing kervolutional layers are named as kervolutional neural networks (KNN).

There is circumstantial evidence that suggests most cells inside striate cortex\footnote{The striate cortex is the part of the visual cortex that is involved in processing visual information.} can be categorized as simple, complex, and hypercomplex, with specific response properties \cite{hubel1968receptive}.
However, the convolutional layers are linear and designed to mimic the behavior of simple cells in human visual cortex \cite{Zoumpourlis:2017vh}, hence they are not able to express the non-linear behaviors of the complex and hypercomplex cells inside the striate cortex.
It was also demonstrated that higher order non-linear feature maps are able to make subsequent linear classifiers more discriminative \cite{lin2015bilinear, Blondel:2016uz, Cui:2017em}.
However, the non-linearity that comes from the activation layers, \eg rectified linear unit (ReLU) can only provide point-wise non-linearity.
We argue that CNN may perform better if convolution can be generalized to patch-wise non-linear operations via kernel trick.
Because of the increased expressibility and model capacity, better model generalization may be obtained.

Non-linear generalization is simple in mathematics, however, it is generally difficult to retain the advantages of convolution, \ie (i) sharing weights (weight sparsity) and (ii) low computational complexity.
There exists several works towards non-linear generalization.
The non-linear convolutional networks \cite{Zoumpourlis:2017vh}  implement a quadratic convolution at the expense of additional $n(n+1)/2$ parameters, where $n$ is the size of the receptive field.
However, the quadratic form of convolution loses the property of "weight sparsity", since the number of additional parameters of the non-linear terms increases exponentially with the polynomial order, which dramatically increases the training complexity.
Another strategy to introduce high order features is to explore the pooling layers. The kernel pooling method in \cite{Cui:2017em} directly concatenates the non-linear terms, while it requires the calculation of non-linear terms, resulting in a higher complexity.

To address the above problems, kervolution is introduced in this paper to extend convolution to kernel space while keeping the aforementioned advantages of linear convolutions.
Since convolution has been applied to many fields, \eg image and signal processing, we expect kervolution will also play an important role in those applications. However, in this paper, we focus on its applications in artificial neural networks.
The contributions of this paper include: (i) via kernel trick, the convolutional operator is generalized to kervolution, which retains the advantages of convolution and brings new features, including the increased model capacity, translational equivariance, stronger expressibility, and better generalization; (ii) we provide explanations for kervolutional layers from feature view and show that it is a powerful tool for network construction; (iii) it is demonstrated that KNN achieves better accuracy and surpasses the baseline CNN.

\section{Related Work}\label{sec:related-work}

As the name indicates, CNN \cite{lecun1989generalization}  employs convolution as the main operation, which is modeled to mimic the behavior of simple cells found in the \textit{primary visual cortex} known as V1 \cite{hubel1962receptive,hubel1968receptive}.
It have been tremendously successful in numerous applications \cite{lecun1989backpropagation,Krizhevsky:2012wl,He:2016ib,Goodfellow2016deep}.
Many strategies have been applied to improve the capability of model generalization.

AlexNet \cite{Krizhevsky:2012wl} proves that the ensemble method "dropout" is very effective for reducing over-fitting of convolutional networks.
The non-saturated rectified linear unit (ReLU) dramatically improves the convergence speed \cite{Krizhevsky:2012wl} and becomes a standard component of CNN.
Network in network (NIN) \cite{Lin:2013wb} establishes micro networks for local patches within the receptive field, each of which consists of multiple fully connected layers followed by non-linear activation functions.
This improves the model capacity at the expense of additional calculation and complex structures.
GoogLeNet \cite{Szegedy:2015ja} increases both depth and width of CNN by introducing the Inception model, which further improves the performance.
VGG \cite{Simonyan:2015ws} shows that deep CNN with small convolution filters ($3\times3$) is able to bring about significant improvement.

ResNet \cite{He:2016ib} addresses the training problem of deeper CNN and proposes to learn the residual functions with reference to the layer input.
This strategy makes CNN easier to optimize and improves regression accuracy with an increased depth.
DenseNet \cite{Huang:2017kg} proposes to connect each layer to every other layer in a feed-forward fashion, which further mitigates the problem  of vanishing-gradient.
The ResNeXt \cite{Xie:2017hu} is constructed by repeating a building block that aggregates a set of transformations with the same topology, resulting in a homogeneous, multi-branch architecture.
It demonstrates the essence of a new dimension, which is the size of the set of transformations.
In order to increase the representation power, SENet \cite{Hu:2017tf} focuses on channels and adaptively recalibrates channel-wise feature responses by explicitly modeling the interdependencies channels.

In recent years, researchers have been paying much attentions on the extension of convolution.
To enable the expressibility of convolution for complex cells, the non-linear convolutional network \cite{Zoumpourlis:2017vh} extends convolution to non-linear space by directly introducing high order terms.
However, as indicated before, this introduces a large number additional parameters and increases the training complexity exponentially.
To be invariant to spatial transformations, the spatial transformer network \cite{Jaderberg:2015vo} inserts learnable modules to CNN for manipulating transformed data.
For the same purpose, deformable convolutional network \cite{Dai:2017vy} adds 2-D learnable offsets to regular grid sampling locations for standard convolution, which enables the learning of affine transforms; while \cite{Henriques:2017te} applies a simple two-parameter image warp before a convolution.
CapsNet \cite{Sabour:2017ts} proposes to replace convolution by representing the instantiation parameters of a specific type of entity as activity vectors via a capsule structure. This opens a new research space for artificial neural networks, although its performance on large dataset is still relatively weak. Decoupled network \cite{Liu:2018vk} interprets convolution as the product of norm and cosine angle of the weight and input vectors, resulting in explicit geometrical modeling of the intra-class and extra-class variation.
To process graph inputs, Spline-CNN \cite{fey2018splinecnn} extends convolution by using continuous B-spline basis, which is parametrized by a constant number of trainable control values.  To decrease the storage,  Modulated CNN \cite{wang2018modulated} extends the convolution operator to binary filters, resulting in easier deployment on low power devices.

The kernel technique in this paper was applied to create non-linear classifiers in the context of optimal margin \cite{boser1992training}, which was later recognized as support vector machines (SVM) \cite{cortes1995support}.
It recently has also been widely applied to correlation filter for improving the processing speed.
For example, the kernelized correlation filter (KCF) \cite{Henriques:2015jy} is proposed to speed up the calculation of kernel ridge regression by bypassing a big matrix inversion, while it assumes that all the data are circular shifts of each other \cite{wang2019kernel}, hence it can only predict signal translation.
To break this theoretical limitation, the kernel cross-correlator (KCC) is proposed
in \cite{Wang:2018vt} by defining the correlator in frequency domain directly, resulting in a closed-form solution with computational complexity of $\mathcal{O}(N\log N)$, where $N$ is the signal length. Moreover, it does not impose any constraints on the training data, thus KCC is useful for other applications \cite{Wang:2017wb,2017arXiv171005502W,Wang:2017wl} and is applicable for affine transform prediction, \eg translation, rotation, and scale.
This theorem is further extended to speed up the prediction of joint rotation-scale transforms in \cite{wang2018correlation}. The above works show that the kernel technique is a powerful tool for obtaining both accuracy and efficiency.

The kernel technique recently has also been applied to artificial neural networks to improve the model performance.
The convolutional kernel network (CKN) \cite{Mairal:2014wb} proposes to learn transform invariance by kernel approximation, where the kernel is used as a tool for learning CNN. Nevertheless, the aim of CKN is not to extract non-linear features and it is only different from CNN in the cost functions.
The SimNets \cite{Cohen:2016gp} proposes to insert kernel similarity layers under convolutional layer.
However, both the similarity templates and filters are needed to be trained and require a pre-training process for initialization, which dramatically increases the complexity.
To capture higher order interactions of features, the kernel pooling \cite{Cui:2017em} is proposed in a parameter-free manner.
This is motivated by the aforementioned idea that higher dimensional feature map produced by kernel functions is able to make subsequent linear classifier more discriminative \cite{Blondel:2016uz}.
However, the kernel extension in the pooling stage is not able to extract non-linear features in a patch-wise way.
Moreover, the additional higher order features are still need to be calculated explicitly, which also dramatically improves the complexity.
To solve these problems, kervolution is defined to generalize convolution via the kernel trick.

\section{Kervolution}

We start from a convolution with output $\mathbf{f}(\mathbf{x})$, \ie
\begin{equation}\label{eq:convolution}
\mathbf{f}(\mathbf{x}) = \mathbf{x}\oplus\mathbf{w},
\end{equation}
where $\oplus$ is the convolutional operator and $\mathbf{x}\in\mathbb{R}^{n}$ is a vectorized input and $\mathbf{w}\in\mathbb{R}^{n}$ is the filter.
Specifically, the $i_{\text{th}}$ element of the convolution output $\mathbf{f}(\mathbf{x})$ is calculated as:
\begin{equation}\label{eq:linear-kernel}
\mathbf{f}_i(\mathbf{x}) = \left<\mathbf{x}_{(i)}, \mathbf{w}\right>,
\end{equation}
where $ \left<\cdot, \cdot\right>$ is the inner product of two vectors and $\mathbf{x}_{(i)}$ is the circular shift of $\mathbf{x}$ by $i$ elements.
We define the index $i$ started from $0$.
The kervolution output $\mathbf{g}(\mathbf{x})$ is defined as
\begin{equation}\label{eq:kervolution}
\mathbf{g}(\mathbf{x}) = \mathbf{x}\otimes\mathbf{w},
\end{equation}
where $\otimes$ is the kervolutional operator.
Specifically, the $i_{\text{th}}$ element of $\mathbf{g}(\mathbf{x})$ is defined as:
\begin{equation}\label{eq:kernel-function-high}
\mathbf{g}_i(\mathbf{x}) = \left<\varphi(\mathbf{x}_{(i)}), \varphi(\mathbf{w})\right>,
\end{equation}
where $\varphi(\cdot):\mathbb{R}^n \mapsto \mathbb{R}^d~(d\gg n)$  is a non-linear mapping function.
The definition \eqref{eq:kernel-function-high} enables us to extract features in a high dimensional space, while its computational complexity is also much higher than \eqref{eq:linear-kernel}. Fortunately, we are able to bypass the explicit calculation of the high dimensional features $\varphi(\mathbf{x}_{(i)})$ via the kernel trick \cite{cortes1995support}, since
\begin{equation}\label{eq:kernel-function}
\left<\varphi(\mathbf{x}_{(i)}), \varphi(\mathbf{w})\right> = \sum_{j} c_j(\mathbf{x}^T_{(i)}\mathbf{w})^j = \kappa(\mathbf{x}_{(i)}, \mathbf{w}),
\end{equation}
where $\kernel{\cdot}{\cdot}:\mathbb{R}^n\times\mathbb{R}^n\mapsto\mathbb{R}$ is a kernel function, whose complexity is normally of $\mathcal{O}(n)$ as same as the inner product of convolution. The coefficient $c_j$ can be determined by the mapping function $\varphi(\cdot)$ or a predefined kernel $\kernel{\cdot}{\cdot}$, \eg the Gaussian RBF kernel, in which the feature dimension $d$ is infinite.
Intuitively, the inner product \eqref{eq:linear-kernel} is a linear kernel, thus convolution is a linear case of kervolution.

The \kerv \eqref{eq:kervolution} retains the advantages of convolution  and brings new features: (i) sharing weights (Section \ref{sec:weight-sparsity}); (ii) equivariance to translation (Section \ref{sec:equivariance}); (iii) increased model capacity and new feature similarity (Section \ref{sec:model-capacity});

\subsection{Sharing Weights}\label{sec:weight-sparsity}
Sharing weights normally mean less trainable parameters and lower computational complexity. It is straightforward that the number of elements in filter $\mathbf{w}$ is not increased according to the definition of \eqref{eq:kernel-function}, thus kervolution keeps the sparse connectivity of convolution.
As a comparison, we take the Volterra series-based non-linear convolution adopted in \cite{Zoumpourlis:2017vh} as an example, the additional parameters of non-linear terms dramatically increase the training complexity, since the number of learnable parameters increases exponentially with the order of non-linearity.
Even an quadratic expression $\mathbf{g}^v(\mathbf{x})$ in \eqref{eq:non-linear-convolution} from \cite{Zoumpourlis:2017vh} is of complexity $\mathcal{O}(n^2)$:
\begin{equation}\label{eq:non-linear-convolution}
\mathbf{g}^v_i(\mathbf{x})= \mathbf{x}^T_{(i)}\mathbf{w}_2\mathbf{x}_{(i)} + \mathbf{w}^T_1\mathbf{x}_{(i)},
\end{equation}
where $\mathbf{w}_1\in\mathbb{R}^{n}$ and $\mathbf{w}_2\in\mathbb{R}^{n\times n}$ are the linear and quadratic filters, respectively.
The quadratic term in \eqref{eq:non-linear-convolution} introduces additional $n(n+1)/2$ parameters ($\mathbf{w}_2$ is an upper triangular matrix).
Instead, a typical non-linear kernel is normally of complexity $\mathcal{O}(n)$, \ie the Gaussian RBF kernel, which is the same with linear kernel \eqref{eq:linear-kernel}, thus kervolution preserves the linear computational complexity $\mathcal{O}(n)$.

Another strategy to introduce higher order features is to explore the pooling layers. For example, the kernel pooling method proposed in \cite{Cui:2017em} directly concatenates the non-linear terms $c_j(\mathbf{x}^T_{(i)}\mathbf{w})^j$ as in \eqref{eq:kernel-function} to the pooling layers. However, this requires explicit calculation of the non-linear terms up to $p$ orders, although it can be approximated by  applying the discrete Fourier transform (DFT) of $p$ times, resulting in a computational complexity of $\mathcal{O}(p\cdot n \log n)$.
Nevertheless, based on the kernel trick, kervolution can introduce any order of non-linear terms yet still with linear complexity.

\begin{figure*}
    \centering
	\includegraphics[width=1.0\linewidth]{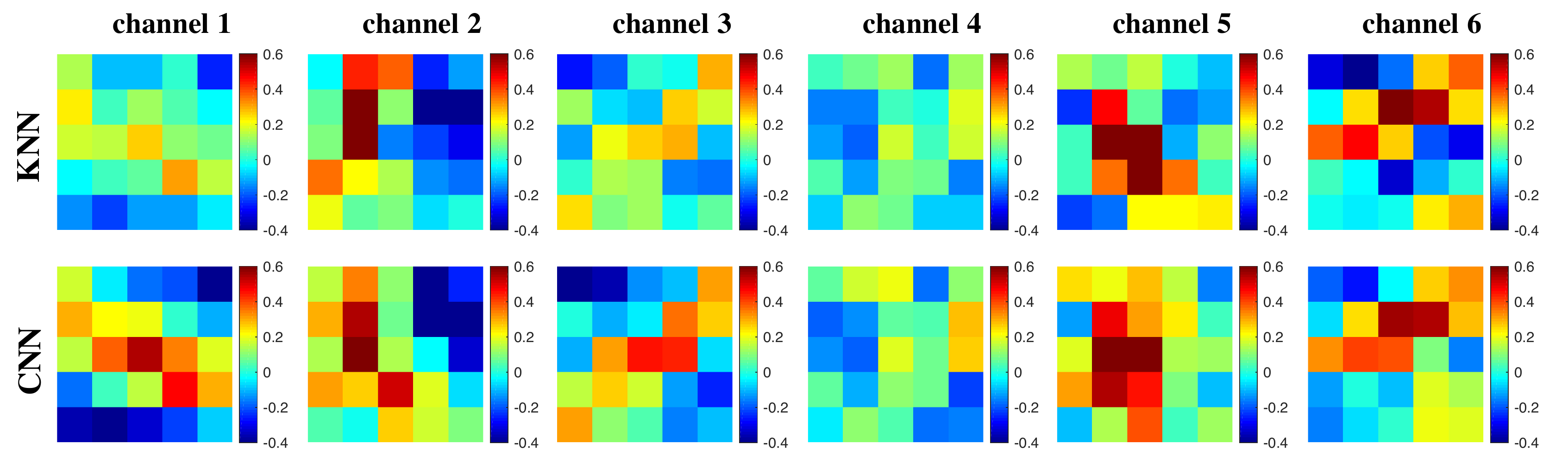}
	\caption{The comparison of learned filters on MNIST from the first layer (six channels and filter size of $5\times5$) of CNN and polynomial KNN. It is interesting that some of the learned filters (\eg channel 4) from KNN are quite similar to CNN. This indicates that part of the kervolutional layer learns linear behavior, which is controlled by the linear part of the polynomial kernel.}
	\label{fig:weight}
\end{figure*}

\subsection{Translational Equivariance}\label{sec:equivariance}

A crucial aspect of current architectures of deep learning is the encoding of invariances.
One of the reasons is that the convolutional layers are naturally equivariant to image translation \cite{Goodfellow2016deep}. In this section, we show that kervolution preserves this important property.
An operator is equivariant to a transform when the effect of the transform is detectable in the operator output \cite{Cohen:2016to}.
Therefore, we have
 $\mathbf{f}_{(j)}(\mathbf{x})=\mathbf{x}_{(j)}\oplus\mathbf{w}$, which means the input translation results in the output translation \cite{Goodfellow2016deep}. Similarly,
\begin{theorem}
	kervolution \eqref{eq:kervolution} is equivariant to translation.
\end{theorem}
\begin{proof}
Assume $\mathbf{g}'(\mathbf{x})=\mathbf{x}_{(j)}\otimes\mathbf{w}$, according to \eqref{eq:kernel-function}, we have
\begin{equation}
\mathbf{g}'_{i}(\mathbf{x})= \kappa(\mathbf{x}_{(i+j)}, \mathbf{w})=\mathbf{g}_{i+j}(\mathbf{x}).
\end{equation}
Therefore, the $i_{\text{th}}$ element of $\mathbf{x}_{(j)}\otimes\mathbf{w}$ is the $(i+j)_{\text{th}}$ element of $ \mathbf{g}(\mathbf{x})$, hence we have
\begin{equation}
 \mathbf{g}_{(j)}(\mathbf{x})=\mathbf{x}_{(j)}\otimes\mathbf{w},
\end{equation}
which completes the proof.
\end{proof}
Note that the translational invariance of CNN is achieved by concatenating pooling layers to convolutional layers \cite{Goodfellow2016deep}, and the translational invariance of KNN can be achieved similarly. This property is crucial, since when invariances are present in the data, encoding them explicitly in an architecture provides an important source of regularization, which reduces the amount of training data required \cite{Henriques:2017te}.
As mentioned in \sref{sec:related-work}, the same property is also presented in \cite{Henriques:2015jy}, which is achieved by assuming that all the training samples are circular shifts of each other \cite{wang2019kernel}, while ours is inherited from convolution.
Interestingly, the kernel cross-correlator (KCC) defined in \cite{Wang:2018vt} is equivariant to any affine transforms (\eg, translation, rotation, and scale), which may be useful for further development of this work.

\begin{table}[!t]
	\begin{center}
		\begin{tabular}{cccc}\toprule
			Method & Convolution  & $L^1$-norm & $L^2$-norm \\\midrule
			None & 99.17 & 99.12    & 99.11   \\
			FGSM & 71.92 & 74.08  & 76.36  \\
			\bottomrule
		\end{tabular}
	\end{center}
	\caption{Test Accuracy (\%) of the white-box FGSM attack with $L^p$-norm kervolution on MNIST. 10K images are randomly selected.}
	\label{tab:attact-distance}
\end{table}

\subsection{Model Capacity and Features}\label{sec:model-capacity}

It is straightforward that the kernel function \eqref{eq:kernel-function} takes \kerv to non-linear space, thus the model capacity is increased without introducing extra parameters.
Recall that CNN is a powerful approach to extract discriminative local descriptors.
In particular, the linear kernel \eqref{eq:linear-kernel} of convolution measures the similarity of the input $\mathbf{x}$ and filter $\mathbf{w}$, \ie the cosine of the angle $\theta$ between the two patches, since
$\left<\mathbf{x},\mathbf{w}\right> = \cos(\theta)\cdot\|\mathbf{x}\|\|\mathbf{w}\|$.
From this point of view, kervolution measures the similarity by match kernels, which are equivalent to extracting specific features \cite{Bo:2009tp}.
We next discuss how to interpret the kernel functions and present a few instances $\kappa(\cdot, \cdot)$ of the kervolutional operator.
One of the advantages of kervolution is that the non-linear properties can be customized without explicit calculation.

\paragraph{$L^p$-Norm Kervolution} The $L^1$-norm in \eqref{eq:1norm-kernel} and $L^2$-norm in \eqref{eq:2norm-kernel} simply measures the  Manhattan and Euclidean distances between input $\mathbf{x}$ and filter $\mathbf{w}$, respectively.
\begin{subequations}\label{eq:distance-kernel}
	\begin{align}
	\kappa_{\text{m}}(\mathbf{x},\mathbf{w}) &= \|\mathbf{x} - \mathbf{w}\|_1,\label{eq:1norm-kernel}\\
	\kappa_{\text{e}}(\mathbf{x},\mathbf{w}) &= \|\mathbf{x} - \mathbf{w}\|_2. \label{eq:2norm-kernel}
	\end{align}
\end{subequations}
Both "distances" of two points involves aggregating the distances between each element.
If vectors are close on most elements, but more discrepant on one of them,
Euclidean distance will reduce that discrepancy (elements are mostly smaller than 1 because of the normalization layers), being more influenced by the closeness of the other elements.
Therefore, the Euclidean kervolution may be more robust to slight pixel perturbation.
This hypothesis is verified by a simple simulation of adversary attack using the fast gradient sign method (FGSM) \cite{Goodfellow:2014tl}, shown in \tref{tab:attact-distance}, where `None' means the test accuracy on clean data.

\paragraph{Polynomial Kervolution}
Although the existing literatures have shown that the polynomial kernel \eqref{eq:polynomial-kernel} works well for the problem of natural language processing (NLP) when $d_p=2$ using SVM \cite{goldberg2008splitsvm},  we find its performance is better when $d_p=3$ in KNN for the problem of image recognition.
\begin{equation}\label{eq:polynomial-kernel}
	\kappa_{\text{p}}(\mathbf{x},\mathbf{w}) = (\mathbf{x}^T\mathbf{w}+c_p)^{d_p} = \sum _{j=0}^{d_p} c_p^{d_p-j}(\mathbf{x}^T\mathbf{w})^j,
\end{equation}
where $d_p~(d_p\in\mathbb{Z}^+)$ extends the feature space to $d_p$ dimensions; $c_p~(c_p\in\mathbb{R}^+)$ is able to balance the non-linear orders (Intuitively, higher order terms play more important roles when $c_p<1$).
As a comparison, the kernel pooling strategy \cite{Cui:2017em} concatenates the non-linear terms $c_j(\mathbf{x}^T\mathbf{w})^{j}$ directly, while they are finally linearly combined by subsequent fully connected layer, which dramaticaly increases the number of learnable  parameters in the linear layer.

To show the behavior of polynomial kervolution, the learned filters of LeNet-5 trained for MNIST are visualized in Figure \ref{fig:weight}, which contains all six channels of the first kervolutional layer using polynomial kernel ($d_p=2, c_p=0.5$).
The optimization process is described in Section \ref{sec:ablation}.
For a comparison, the learned filters from CNN are also presented.
It is interesting that some of the learned filters of KNN and CNN are quite similar, \eg channel 4, which means that part of the capacity of KNN learns linear behavior as CNN.
This verifies our understanding of polynomial kernel, which is a combination of linear and higher order terms. This phenomenon also indicates that polynomial kervolution introduces higher order feature interaction in a more flexible and direct way than the existing methods.

\paragraph{Gaussian Kervolution}
The Gaussian RBF kernel \eqref{eq:gaussian-kernel} extends kervolution to infinite dimensions.
\begin{equation}\label{eq:gaussian-kernel}
\kappa_{\text{g}}(\mathbf{x},\mathbf{w}) =\exp(-\gamma_g\|\mathbf{x}-\mathbf{w}\|^2),
\end{equation}
where  $\gamma_g~(\gamma_g\in\mathbb{R}^+)$ is a hyperparameter to control the smoothness of decision boundary. It extends kervolutoin to infinite dimensions because of  the $i$-degree terms in \eqref{eq:gaussian-series}.
\begin{equation}\label{eq:gaussian-series}
\kappa_{\text{g}}(\mathbf{x},\mathbf{w})=C\sum_{i=0}^{\infty}\frac{(\mathbf{x}^T\mathbf{w})^i}{i!},
\end{equation}
where $C=\exp\left(-\frac{1}{2}\left(\|\mathbf{x}\|^2+\|\mathbf{w}\|^2\right)\right)$ if $\gamma_g=\frac{1}{2}$.

The expression \eqref{eq:gaussian-series} is helpful for our intuitive understanding, while the recent discovery reveals more information.
It is shown in \cite{Bo:2010vi} that the Gaussian kernel and its variants are able to measure the similarity of gradient based patch-wise features, \eg SIFT \cite{Anonymous:2004uq} and HOG \cite{Dalal:2005gq}.
This provides a unified way to generate a rich, diverse visual feature set \cite{Gehler:2009hj}. However, instead of using the hand-crafted features as kernel SVM, with KNN, we are able to inherit the substantial achievements based on kernel trick while still taking advantage of the great generalization ability of neural networks.

\subsection{Kervolutional Layers and Learnable Kernel}
Similar to a convolutional layer, the operation of a kervolutional layer is slightly different from the standard definition \eqref{eq:kervolution} in which $\mathbf{x}_{(i)}$ becomes a 3-D patch in a sliding window on the input.
To be compatible with existing works, we also implement all popular available structures of convolution in CNN library \cite{paszke2017automatic} for kervolution, including the input and output channels, input padding, bias, groups (to control connections between input and output), size, stride, and dilation of the sliding window.
Therefore, the convolutional layers of all existing networks can be directly or partially replaced by kervolutional layers, which makes KNN inherit all the the existing achievements of CNN, \eg network architectures \cite{Krizhevsky:2012wl,He:2016ib} and their numerous applications \cite{Ren:2017kt}.

With kervolution, we are able to extract specific type of features without paying attention to the weight parameters.
However, as aforementioned, we still need to tune the hyperparameters for some specific kernels, \eg the balancer $c_p$ in polynomial kernel, the smoother $\gamma_g$ in Gaussian RBF kernel.
Although we noticed that the model performance is mostly insensitive to the kernel hyperparameters, which is presented in \sref{sec:hyperparameters}, it is sometimes troublesome when we have no knowledge about the kernel.
Therefore, we also implement training the network with learnable kernel hyperparameters based on the back-propagation \cite{rumelhart1988learning}.
This slightly increases the training complexity theoretically, but in experiments we found that this brings more flexibility and the additional cost for training several kernel parameters is negligible, compared to learning millions of parameters in the network.
Taking the Gaussian kervolution as an example, the gradients are computed as:
\begin{subequations}
	\begin{align}
	\frac{\partial }{\partial \mathbf{w}}
	\kappa_{\text{g}}(\mathbf{x},\mathbf{w}) &=2\gamma_g(\mathbf{x}-\mathbf{w})\kappa_{\text{g}}(\mathbf{x},\mathbf{w}),\\
	\frac{\partial }{\partial \gamma_g}
	\kappa_{\text{g}}(\mathbf{x},\mathbf{w}) &=-\|\mathbf{x}-\mathbf{w}\|^2\kappa_{\text{g}}(\mathbf{x},\mathbf{w}).
	\end{align}
\end{subequations}
Note that the polynomial order $d_p$ is not trainable because of the integer limitation, since the real exponent may produce complex numbers, which makes the network complicated.

\begin{figure*}
	\centering
	\includegraphics[width=1.0\linewidth]{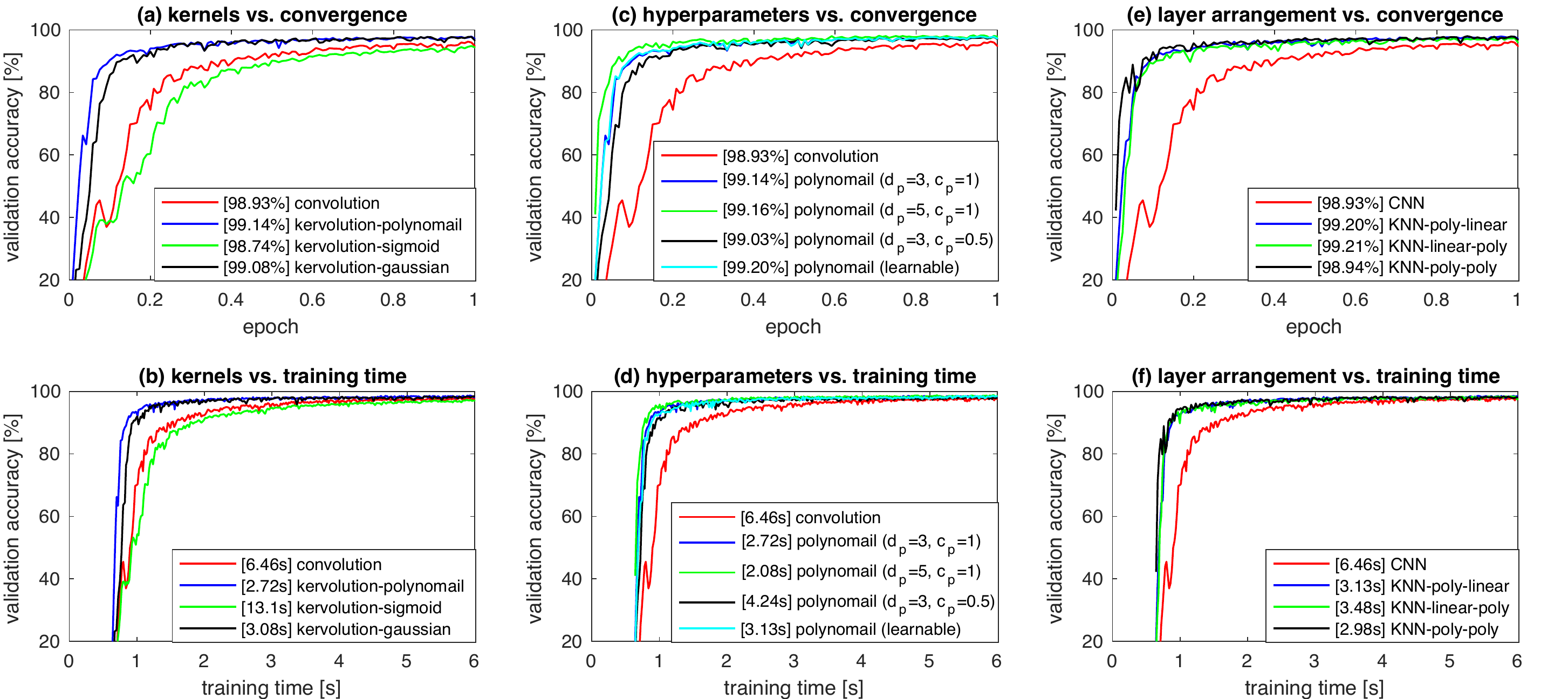}
	\caption{The influences of kervolution on the convergence rate. The best validation accuracy (20 epochs) and training time to target accuracy ($98\%$) are displayed in brackets within the legend of each figure, respectively. (a) and (b) demonstrate that the kernel functions have a significant impact on the convergence speed. (c) and (d) demonstrate that the hyperparameters of kervolutional layer play less important roles than kernels. (e) and (f) shows the effects of arrangement of kervolutional layers.}
	\label{fig:hyperparameter}
\end{figure*}

\section{Ablation Study}\label{sec:ablation}

This section explores the influences of the kernels, the hyperparameters, and combination of kervolutional layers using LeNet-5 and MNIST \cite{Lecun:1998hy}.
To eliminate the influence of other factors, all configurations are kept as the same.
The accuracy of modern networks on MNIST has been saturated, thus we adopt the evaluation criteria proposed in DAWNBench \cite{Coleman:ue} that jointly considers the computational efforts and precision.
It measures the total training time to a target validation accuracy (98\%), which is a trade-off between efficiency and accuracy.
In all the experiments of this section, we apply the stochastic gradient descent (SGD) method for training, where a mini-batch size of 50, a momentum of 0.9, an initial learning rate of 0.003, a multiplicative factor of 0.1, and a maximum epoch of 20 with milestones of $[10, 15]$ are adopted.
Our algorithm is implemented based on the PyTorch library \cite{paszke2017automatic}.
All tests are conducted on a single Nvidia GPU of GeForce GTX 1080Ti.
The reported training time does not include the testing and checkpoint saving time.

\subsection{Kernels}
Following the ablation principle, we only replace the convolutional layers of LeNet-5 by kernvolutional layer using three kernel functions, \ie polynomial kernel of $d_p=3, c_p=1$ in \eqref{eq:polynomial-kernel},
Gaussian kernel of $\gamma_g=1$ in \eqref{eq:gaussian-kernel}, and also sigmoid kernel $\kappa_{\text{s}}(\mathbf{x},\mathbf{w}) = \tanh(\mathbf{x}^T\mathbf{w})$.
As shown in \fref{fig:hyperparameter} (a) and (b), although the computational complexity of non-linear kernels is slightly higher than that of linear kernel (convolution), the polynomial and Gaussian KNN are still able to converge to a validation accuracy of $98\%$ more than $2\times$ faster than the original CNN.
However, the convergence speed of sigmoid KNN is $2\times$ slower than that of CNN, which indicates that the kernel functions are crucial and have a significant impact on performance.
Thanks to the wealth of traditional methods, we have many other useful kernels \cite{smola1998learning}, although we cannot test all of them in this paper. The $L^1$ and  $L^2$-norm KNN achieve an accuracy of 99.05\% and 99.19\%, respectively, but we omit them in \fref{fig:hyperparameter} (a) and (b)  because they nearly coincide with the polynomail curve.

\subsection{Hyperparameters}\label{sec:hyperparameters}
From the above analysis, we credit the significant improvements of convergence speed to the usage of different kernels.
This part verifies this assumption and further explores the influences of kernel hyperparameters.
The polynomial kervolution \eqref{eq:polynomial-kernel} with two hyperparameters, non-linear order $d_p$ and balancer $c_p$, is selected.
As shown in Figure \ref{fig:hyperparameter} (c) and (d), the convergence speed and validation accuracy of polynomial KNN using different hyperparameters are pretty similar to that of Figure \ref{fig:hyperparameter} (a) and (b), which indicates that KCC is less sensitive to the kernel hyperparameters.

It is also noticed that the KNN with learnable kernel parameters achieves the best precision (99.20\%) in this group, although it slightly increases the training time compared to KNN ($d_p=3,c_p=1$). However, the cost is justifiable since it saves the hyperparameter tuning process and the convergence is still much faster than the baseline CNN.

\subsection{Layer Arrangement}\label{sec:layer-arrangement}
This part explores the influences of the placement of kervolutional layers.
Thanks to the simplicity of LeNet-5 (two convolutional layers), we can test all possible configurations of layer arrangement, \ie "conv-conv", "kerv-conv", "conv-kerv", and "kerv-kerv".
As shown in Figure \ref{fig:hyperparameter} (e) and (f), where the polynomial kernel ($d_p=3,c_p=1$) is adopted, KNN still brings faster convergence.
One interesting phenomenon is that the architecture of "kerv-conv" achieves better precision but slower convergence than "conv-kerv" (we run multiple times and the results are similar).
This indicates that the sequence of kervolutional layers has an impact on performance, although the model complexity is the same.
One side effect is that we may need to make some efforts to adjust the layer sequence for deeper KNNs.
It is also noticed that the architecture of `kerv-kerv' achieves the fastest convergence but only with a comparable validation accuracy to CNN.
We argue that this is caused by the over-fitting problem, since its final training loss is very close to others ($\approx0.01$), which means that the model capacity of double kervolutional layers together with the activation and max pooling layers is too large for the MNIST dataset.

\subsection{Removing ReLU}
As mentioned in \sref{sec:introduction}, the non-linearity of CNN mainly comes from the activation (ReLU) and max pooling layers.
Intuitively, KNN may be able to achieve same model performance without activation or max pooling layers.
To this end, we simply remove all the activation layers of LeNet-5 and replace the max pooling by average pooling layers, which means that all the non-linearity comes from the kervolutional layers.
Without surprise, the CNN only achieves an accuracy of 92.22\%, which is far from the target accuracy of 98\%, hence the training time comparison figure is omitted.
However, the KNN of "gaussian-polynomial" and "polynomial-polynomial" both achieve an accuracy of 99.11\%, which further verifies the effectiveness of kervolution.
In another sense, the strategy of removing the activation layers is one of the solutions to the aforementioned over-fitting problem in \sref{sec:layer-arrangement}, although we need more investigations to find the best architectures for KNN.

\begin{table}
	\begin{center}
		\begin{tabular}{ccc}\toprule
			Network & CIFAR-10 & CIFAR-100\\\midrule
			CNN  \cite{Huang:2016vd} & 13.63    & 44.74 \\
			KNN & {10.85}  & {37.12}  \\
			\bottomrule
		\end{tabular}
	\end{center}
	\caption{Validation error (\%) of ResNets on CIFAR-10 and CIFAR-100 without data augmentation.}
	\label{tab:cifar-}
\end{table}

\begin{table}
	\begin{center}
		\begin{tabular}{ccccc}\toprule
			\multirow{2}{*}{Architecture}  & \multicolumn{2}{c}{CIFAR-10+}   & \multicolumn{2}{c}{CIFAR-100+}\\
			& CNN  &  KNN & CNN   &  KNN \\\midrule
			GoogLeNet \cite{dubey2018pairwise} & 13.37 & {5.16} & 26.65 & {20.84}\\
			ResNet \cite{He:2016ib}  & 6.43 & {4.69}  & 27.22 & {22.49}\\
			DenseNet \cite{Huang:2017kg} & 5.24 & {5.08}   & {24.42} &  24.92\\
			\bottomrule
		\end{tabular}
	\end{center}
	\caption{Validation error (\%) on CIFAR-10 and CIFAR-100 on different architectures with data augmentation.}
	\label{tab:cifar+}
\end{table}

\section{Performance}
This section aims to demonstrate the effectiveness of deep KNN on larger datasets.
In practice, the network architecture has a significant impact on performance.
Since the modern networks are so deep and kervolution provides many possibilities via different kernels, we cannot perform exhaustive tests to find the best sequence of kervolutional layers.
Hence, we construct KNN based on several existing architectures by mainly changing the first convolutional layers to kervolutional layers.
Other factors, such as data augmentation and optimizers, are kept as their original configurations.
As discovered in Section \ref{sec:ablation}, this may not be the best configuration, while it can demonstrate the effectiveness of kervolution.

The CIFAR experiments in this section are conducted in a single GPU of Nvidia GeForce GTX 1080Ti while we employ four Nvidia Tesla M40 in the ImageNet experiments.
The polynomial kervolutional layer in this section adopts the learnable balancer $c_p$ with power $d_p=3$.

\subsection{CIFAR}\label{sec:cifar}

The CIFAR-10 and CIFAR-100 \cite{krizhevsky2009learning} datasets consist of colored natural images with $32\times32$ pixels in 10 and 100 classes, respectively.
Each dataset contains $50\kilo$ images for training and $10\kilo$ for testing. In the testing procedure, only the single view of the original image is evaluated.

\begin{table}[!t]
	\begin{center}
		\begin{tabular}{ccc}  \toprule
			Hyperparm & $d_p=3$  & $d_p=5$\\  \midrule
			$c_p=0$   & 4.78 & 5.42  \\
			$c_p=1$   & 4.60 & 5.36  \\
			learnable $c_p$  & 4.76 & 4.73 \\
			\bottomrule
		\end{tabular}
	\end{center}
	\caption{Validation error (\%) of polynomial KNN using ResNet-32 on CIFAR-10+ using different hyper parameters.}
	\label{tab:cifar-poly-acc}
\end{table}

\begin{table}[!t]
	\begin{center}
		\begin{tabular}{ccc}\toprule
			Hyperparm & $d_p=3$  & $d_p=5$\\\midrule
			$c_p=0$    & 0.83   & 1.41  \\
			$c_p=1 $    & 1.41   & 1.46  \\
			learnable $c_p$  & 0.86 & 0.79 \\
			\bottomrule
		\end{tabular}
	\end{center}
	\caption{Training time (\hour)  of polynomial KNN using ResNet-32 on CIFAR-10+ using different hyper parameters.}
	\label{tab:cifar-poly-time}
\end{table}

The proposed KNNs are first evaluated without data augmentaion using the architecture of ResNet. We construct and train ResNet-110 following the architecture of \cite{He:2016ib} with cross-entropy loss. The stochastic gradient descent (SGD) is adopted with momentum of 0.9.
We train the networks for 200 epochs with a mini-batch size of 128.
The learning rate decays by 0.1 at the 75, 125, and 150 epochs; the weight decay stays at $5\times10^{-4}$.
The validation error of KNN as well as the best performance of baseline CNN from \cite{Huang:2016vd} are presented in Table \ref{tab:cifar-}.
It is interesting that KNN outperforms CNN with a significant improvements on the CIFAR dataset.

We perform more experiments using the data augmentation techniques, and the datasets are denoted as `CIFAR-10+' and 'CIFAR-100+', respectively. The KNNs are constructed following the architecture of GoogLeNet \cite{Szegedy:2015ja} and DenseNet-40-12 \cite{Huang:2017kg}.
Data augmentation is applied following the configuration in ResNet \cite{He:2016ib}, including horizontal flipping with a probability of 50\%, reflection-padding by 4 pixels, and random crop with size $32\times32$.

Different from ResNet, we train DenseNet-40-12 following its original configuration \cite{Huang:2017kg}, \ie SGD with batch size 64 for 300 epochs. The initial learning rate is set to 0.1, and is divided by 10 at 50\% and 75\% of the total number of iterations.
Table \ref{tab:cifar+} lists the performance of KNN and baselines from \cite{He:2016ib,Huang:2017kg,dubey2018pairwise}.
We cannot see a significant improvement on DenseNet, which indicates that polynomial kervolution may not suit for fully connected architecture.

We further demonstrate the sensitivity to the kernel hyperparameters.
Table \ref{tab:cifar-poly-acc} lists the validation errors of KNN on CIFAR-10+ with polynomial kernel of $d_p=3,5$ and $c_p=0,1$ using the architecture of ResNet-32. The performance with learnable balancer of $c_p$ is also given for comparison.
As suggested by \cite{Coleman:ue}, their training time to an accuracy of 94\%  is measured and presented in Table \ref{tab:cifar-poly-time}.
It is interesting that, the configuration of $d_p=3, c_p=1$ achieves the best accuracy, while $d_p=5$ with learnable $c_p$ requires the least training time. The networks with learnable kernel achieve the best overall performance by jointly considering the training time and accuracy. This indicates that the learnable kernel technique can produce compromised performance without tuning parameters.

\subsection{ImageNet}
The ILSVRC 2012 classification dataset \cite{deng2009imagenet} is composed of $1.2 \mega$ images for training and $50 \kilo$ images for validation in $1000$ classes. For fair comparison, we apply the same data augmentation as described in \cite{He:2016ib,he2016identity}, where the single-crop and 10-crop at a size of $224\times 224$ are applied for testing.

We select four versions of ResNet \cite{He:2016ib}, including ResNet-18, ResNet-34, ResNet-50 and ResNet-101, as the baselines.
The kervolutional layer is applied with a polynomial kernel ($d=3,c_p=2$). All the networks are trained using the stochastic gradient descent (SGD) method for 100 epochs with a batch size of 256. The learning rate is set to 0.1, and is reduced every 30 epochs.
Also, a weight decay of $10^{-4}$ and a momentum of $0.9$ without dampening are employed.
In our experiments, the best performance of ResNet cannot be achieved in limited training time.
To guarantee a fair comparison, the results of ResNet which have the best accuracy in \cite{fb2016, He:2016ib, he2016identity, Huang:2017kg} are chosen.
We report the single-crop and 10-crop validation error on ImageNet in Table \ref{tab:imagenet}, where the performance of KNN is the average of five running.

In Table \ref{tab:imagenet}, top-1 errors using ResNet-18/34/50/101 are reduced by 0.5\%, 0.41\%, 0.29\%, and 0.7\% in single-crop testing and 0.45\%, 0.75\%, 0.80\%, and 0.83\% in 10-crop testing, respectively.
For top-5 errors, KNN outperform corresponding ResNets by 0.43\%, 0.24\%, 0.23\%, and 0.28\% in single-crop testing, and 0.39\%, 0.68\%, 0.74\%, and 0.87\% in 10-crop testing, respectively. It is noticed that simple replacement of the convolutional layer in ResNet leads to obvious improvements. We believe that more customized network architecture as well as extensive hyperparameter searches can further improve the performance on ImageNet.

\begin{table}
	\centering
	\begin{center}
		\begin{tabular}{ccc}\toprule
			Network& Top-1  & Top-5\\\midrule
			ResNet-18   & 30.24 / 27.88 & 10.92 / 9.42 \\
			KNN-18 & \textbf{29.74 / 27.43}  & \textbf{10.49 / 9.03}  \\
			\midrule
			ResNet-34  & 26.70 / 25.03 & 8.58 / 7.76 \\
			KNN-34 & \textbf{26.29 / 24.28}  & \textbf{8.34 / 7.08}  \\
			\midrule
			ResNet-50   & 23.85 / 22.85 & 7.13 / 6.71 \\
			KNN-50 & \textbf{23.56 / 22.05}  & \textbf{6.90 / 5.97}  \\
			\midrule
			ResNet-101   & 22.63 / 21.75 & 6.44 / 6.05 \\
			KNN-101 & \textbf{21.93 / 20.92}  & \textbf{6.16 / 5.18}  \\
			\bottomrule
		\end{tabular}
	\end{center}
	\caption{The top-1 and top-5 validation errors on ImageNet, with single-crop / 10-crop testing, respectively.}
	\label{tab:imagenet}
\end{table}

\section{Discussion}

\paragraph{Kernel}
Different from convolution, which can only extract linear features, kervolution is able to extract customized patch-wise non-linear features, which makes KNN much more flexible.
It is demonstrated that the higher order terms make the subsequent linear classifier more discriminant, while this does not increase the computational complexity.
However, we have only tested several kernels, \eg polynomial and Gaussian, which may not be optimal.
It is obvious that the kernel functions and their hyperparameters can be task-driven and more investigations are necessary.

\paragraph{Training}
It is also noticed that the training can be unstable when a network contains too much non-linearity, this is because the model complexity is too high for a specific task,
which can be simply solved by reducing the number of kervolutional layers.
More investigations on searching appropriate non-linearity for a specific task is challenging.

\paragraph{Architecture}
We have only applied kervolution to the existing architectures, \eg ResNet. While this is not optimal, especially when considering that the mechanism of deep architectures is still unclear \cite{kuo2016understanding}. It is obvious that the performance of kernvolution is dependent on the architecture. One of the interesting challenges for future work is to investigate the relationship between the architecture and kervolution.

\section{Conclusion}
This paper introduces the kervolution to generalize convolution to non-linear space, and extends convolutional neural networks to kervolutional neural networks.
It is shown that kervolution not only retains the advantages of convolution, \ie sharing weights and equivalence to translation, but also enhances the model capacity and captures higher order interactions of features, via patch-wise kernel functions without introducing additional parameters. It has been demonstrated that, with careful kernel chosen, the performance of CNN can be significantly improved on MNIST, CIFAR, and ImageNet dataset via replacing convolutional layers by kervolutional layers. Due to the large number of choices of kervolution, we cannot perform a brute force search for all the possibilities, while this opens a new space for the construction of deep networks. We expect the introduction of kervolutional layers in more architectures and extensive hyperparameter searches can further improve the performance.

{\small
\bibliographystyle{ieee_fullname}
\bibliography{egbib,papers}
}

\includepdf[pages=-]{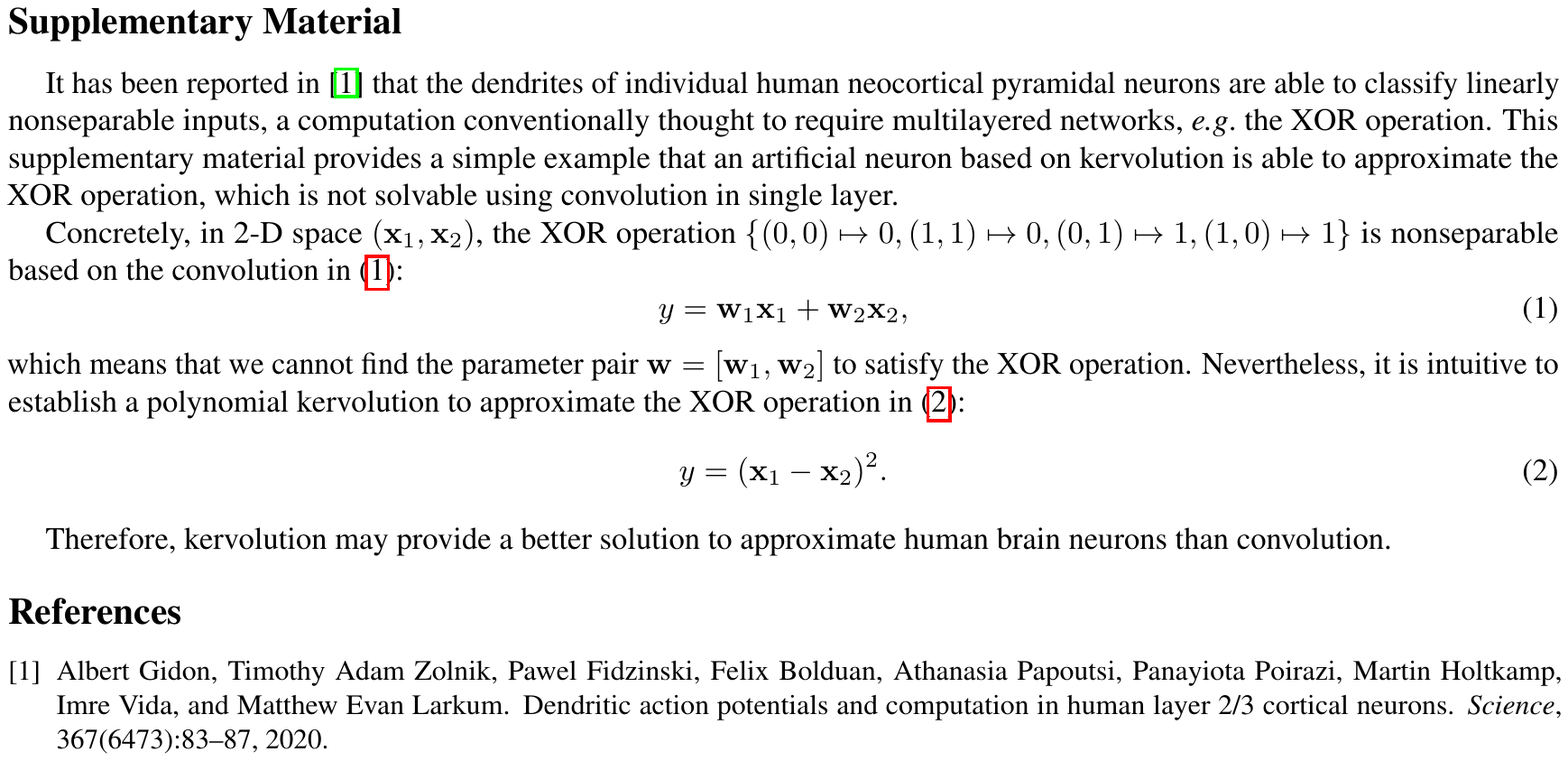}

\end{document}